\documentclass[twoside,reqno,twocolumn]{article}

\usepackage[letterpaper]{geometry}

\usepackage{ltexpprt}
\usepackage{hyperref}
\usepackage{adjustbox}
\usepackage{amsmath}
\usepackage{lipsum}
\usepackage{float}
\usepackage{placeins}
\usepackage{microtype}
\usepackage[noend]{algpseudocode}
\usepackage{subfig, float, xspace}
\usepackage{graphicx}
\usepackage{balance}
\usepackage[utf8]{inputenc}
\usepackage{algorithm}
\DeclareMathOperator{\cT}{\mathcal{T}}
\usepackage[colorinlistoftodos,textsize=tiny]{todonotes}
\usepackage{dsfont}
\newcommand\numberthis{\addtocounter{equation}{1}\tag{\theequation}}

\usepackage{times}
\usepackage[T1]{fontenc}
\begin{document}
\newcommand\relatedversion{}
\renewcommand\relatedversion{\thanks{The full version of the paper can be accessed at \protect\url{https://arxiv.org/abs/1902.09310}}} 

\title{\Large CPTAM: Constituency Parse Tree Aggregation Method}
\author{Adithya Kulkarni\thanks{The first two authors contributed equally to this work.}\qquad Nasim Sabetpour$^*$\qquad Alexey Markin\qquad Oliver Eulenstein\qquad Qi Li \\ {\{aditkulk, nasim, amarkin, oeulenst, qli\}@iastate.edu}
\\Iowa State University, Ames, Iowa, USA} 
\maketitle
\def\algbackskip{\hskip-\ALG@thistlm}

\fancyfoot[R]{\scriptsize{Copyright \textcopyright\ 2022 by SIAM\\
Unauthorized reproduction of this article is prohibited}}

\begin{abstract} Diverse Natural Language Processing tasks employ constituency parsing to understand the syntactic structure of a sentence according to a phrase structure grammar. Many state-of-the-art constituency parsers are proposed, but they may provide different results for the same sentences, especially for corpora outside their training domains. This paper adopts the truth discovery idea to aggregate constituency parse trees from different parsers by estimating their reliability in the absence of ground truth. Our goal is to consistently obtain high-quality aggregated constituency parse trees. We formulate the constituency parse tree aggregation problem in two steps, structure aggregation and constituent label aggregation. Specifically, we propose the first truth discovery solution for tree structures by minimizing the weighted sum of Robinson-Foulds ($RF$) distances, a classic symmetric distance metric between two trees. Extensive experiments are conducted on benchmark datasets in different languages and domains. The experimental results show that our method, CPTAM, outperforms the state-of-the-art aggregation baselines. We also demonstrate that the weights estimated by CPTAM can adequately evaluate constituency parsers in the absence of ground truth. \end{abstract}

\textit{Keywords-} Constituency parse tree, Truth discovery, Optimization

\section{Introduction}
The constituency parse trees (CPTs) display the syntactic structure of a sentence using context-free grammar. CPTs divide the input sentence into phrase structures that belong to a specific grammar category. The available state-of-the-art constituency parsers use different parsing techniques. They are leveraged in various NLP applications like Question Answering, Information Extraction, and word-processing systems. However, due to multiple limitations, the state-of-the-art constituency parsers may make errors, and different constituency parsers may give different results for the same sentence.

\begin{table*} [ht]  \small
\centering 
\begin{adjustbox}{width=160mm,center}
\begin{tabular}{l|c|c|c|c|c|c}
\hline
& \multicolumn{1}{|p{2cm}|}{\centering Penn \\ Treebank-3}
& \multicolumn{1}{|p{2cm}|}{\centering OntoNotes \\ (English)}
& \multicolumn{1}{|p{2cm}|}{\centering OntoNotes \\ (Chinese)}
& \multicolumn{1}{|p{2cm}|}{\centering French \\ Treebank}
& \multicolumn{1}{|p{2cm}|}{\centering TIGER \\ Corpus}
& \multicolumn{1}{|p{2cm}}{\centering Genia}\\

\hline 
Agreement of all parsers & 1.32 & 6.65 & 23.45 & 0.91 & 3.23 & 0.21
\\
Agreement of two parsers & 45.33 & 55.14 & 31.73 & 7.69 & 2.94 & 16.25
\\
No agreement & 53.35 & 38.21 & 44.82 & 91.40 & 93.83 & 83.54
\\
\hline 
\end{tabular}%
\end{adjustbox}
\caption{Percentage of the sentences that different parsers agree on the tree structure}
\label{table: Intro}
\end{table*}
The conflicts among parsers can confuse users on the parser to use for the downstream tasks, as the performance of different parsers can vary significantly on different domains and languages. No parser can consistently achieve the best results on all datasets, and it is costly and impractical for users to obtain ground truth parsing results. Table \ref{table: Intro} shows the percentage of agreement among the structure of the parsers' outputs on six benchmark datasets, including Penn Treebank-3 \cite{taylor2003penn}, OntoNotes (English and Chinese) \cite{pradhan2009ontonotes}, Genia \cite{ohta2002genia}, French Treebank \cite{abeille2019corpus}, and TIGER Corpus \cite{brants2004tiger}. We execute four parsers including Berkeley \cite{kitaev-klein-2018-constituency}, CoreNLP \cite{manning2014stanford}, AllenNLP \cite{gardner-etal-2018-allennlp}, and Hanlp \cite{hanlp2}, for the English datasets, and three parsers, namely Berkeley, CoreNLP, and Hanlp, for the non-English datasets. On the Penn Treebank-3 dataset, it can be observed that all the parsers agree only on 1.32\% of the sentences. A similar observation can be made for other datasets. 

To overcome these challenges, we aim to construct a CPT that performs consistently well to represent the constituency grammar of the sentence in the absence of ground truth. Intuitively, such CPTs can be constructed by aggregating the parsing results from the state-of-the-art parsers to keep the common structure among the parsers' output and resolve their conflicts. Therefore, we propose to aggregate CPTs through the truth discovery idea. 

Truth discovery has been proposed to conduct weighted aggregation for various applications, where the weights reflect the source reliabilities and are inferred from the data itself without the knowledge of the ground truth \cite{li2016survey}. Truth discovery algorithms \cite{CRH14,sabetpour2020optsla,sabetpour2021truth,yin2008truth} witness successful applications on the aggregation of categorical and numerical data types. However, the aggregation of the tree data type has never been investigated in this domain.  

There are tree aggregation methods proposed in the phylogenetic domain \cite{adams1972consensus,MARGUSH1981239,sokal1981taxonomic,goloboff2002semi} and ensemble methods for parsers \cite{surdeanu2010ensemble}. The issue with studies in the phylogenetic domain is that many assumptions are not applicable for CPTs, and none of them consider the constituent label aggregation. The ensemble methods use ground truth to evaluate the quality of the weak learners, whereas, for our task, the aggregation needs to be conducted in the absence of ground truth.

In this paper, we adopt the truth discovery framework to aggregate CPTs from different input parsers by estimating the parsers' reliability without ground truth. We formulate the constituency parse tree aggregation problem in two steps, structure aggregation and constituent label aggregation. In the structure aggregation step, the key challenges are measuring the distance between trees and constructing the aggregated tree that can minimize that distance. We adopt the Robinson-Foulds ($RF$) distance, a symmetric difference metric for trees \cite{robinson1981comparison}, to measure the distance between the aggregated tree and input CPTs. In practice, we propose an efficient algorithm that can construct the optimal aggregated tree in near-linear time and provide theoretical proofs. We adopt the same truth discovery framework in the constituent label aggregation step.

Extensive empirical studies demonstrate that the proposed Constituency Parse Tree Aggregation Model (CPTAM) can consistently obtain high-quality aggregated CPTs across different languages and domains. Specifically, we apply the most widely used constituency parsers as the input parsers on six corpora from English, Chinese, French and German languages, and from general domains and bio-medical domains. Our experimental results validate that there is no single parser that can achieve the best results across the corpora. CPTAM can consistently obtain high-quality results and significantly outperforms the aggregation baselines. We further examine the estimated weights for the parsers and illustrate that the weight estimation can correctly reflect the quality of each parser in the absence of ground truth. 
 
In summary, our main contributions are:
\begin{itemize}
    \item We identify the pitfalls and challenges in data with tree structures for the task of truth discovery.
    \item We adopt Robinson-Foulds ($RF$) distance to measure the differences among data with tree structures.
    \item We construct the best aggregation trees by solving an optimization problem and derive the theoretical proofs for the correctness and the efficiency of the algorithm.
    
    \item We test the proposed algorithm on real-world datasets,
and the results clearly demonstrate the advantages of the approach in finding the accurate tree structures from the multi-sourced input.
\end{itemize}

\section{Related Works}
We summarize the related works in three categories as below.
\subsection{Truth Discovery} 
Truth discovery aims to resolve the conflicts from multiple sources~\cite{li2016survey}. One line of work applies probabilistic methods to model the workers’ labeling behavior in crowdsourcing platforms~\cite{dong2014knowledge,ma2015faitcrowd,li2014confidence}. Another line of work formulates optimization frameworks that seek to minimize the weighted distance between the source and aggregated results and estimate the source reliability \cite{CRH14,yNeurIPS12}. Recent truth discovery methods consider different applications such as aggregation of sequential labels \cite{sabetpour2021truth,simpson-gurevych-2019-bayesian,nguyen2017aggregating} and aggregation of time series data \cite{yao2018online,li2015discovery, zhi2018dynamic}.

Most of the available truth discovery methods mainly focus on the numerical and categorical data \cite{li2016survey}, and none of them consider tree structure. Furthermore, the distance measurements introduced in previous works do not support the tree structure. However, the problem of how to aggregate information from trees into one representative tree has been of great importance for various applications \cite{ren2005survey}.

\subsection{Phylogenetic Tree Aggregation Problem} \label{related : phylo}

The tree aggregation problem has been studied in the \textit{phylogenetic} domain, where trees are branching diagrams showing the evolutionary relationships among biological species or other taxa \cite{felsenstein2004inferring}. The taxa can be described through different types of data (e.g., morphological or biomolecular). Since the inference of phylogenetic trees is an immensely complex problem, practitioners often perform many tree estimation runs with the same or different phylogenetic inference methods. The estimated trees are aggregated using \emph{consensus tree} techniques~\cite{bininda2002super,bryant2003classification}.

A variety of methods have been developed for phylogenetic tree aggregation \cite{adams1972consensus,bryant2003classification}. Some methods conduct aggregation through simple heuristics when the aggregated tree only contains branches with a certain percentage of agreement, such as the majority rule consensus \cite{MARGUSH1981239}, strict consensus \cite{bremer1990combinable}, semi-strict consensus \cite{goloboff2002semi}, and the greedy consensus \cite{bryant2003classification}. Further, \emph{supertree} and \emph{median tree} approaches have been extensively explored to compute fully binary aggregated trees \cite{bininda2004phylogenetic}. Such methods typically seek an output tree that minimizes the overall distance to the input trees. Since the mentioned methods are introduced in \textit{phylogenetic} domain, they do not consider the characteristics of parse trees.

\subsection {Ensemble Trees}
Tree ensemble methods such as Random Forest~\cite{probst2019hyperparameters} or Boosted trees~\cite{de2007boosted} are not suitable for our needs since these methods ensemble on the classification decisions instead of constructing an aggregation tree. 

There are multiple ensemble models for the parsing of syntactic dependencies in the literature, aiming to construct aggregation trees \cite{surdeanu2010ensemble,kuncoro-etal-2016-distilling}. These parsing tree ensemble methods are commonly categorized into two groups. The first group aggregates the base parsers at training time~\cite{nivre2008integrating,attardi2009reverse,surdeanu2010ensemble}. The second group aggregates the independently trained models at the prediction time~\cite{sagae2006parser,hall2010single,kuncoro-etal-2016-distilling}. One of the common approaches in these ensemble methods is to find the maximum spanning tree (MST) for the directed weighted graph to obtain the optimal dependency structure. Unlike our proposed task, all these methods rely on the ground truth to estimate the parsers' reliability.

\section{Preliminaries} \label{Preliminaries}

This section briefly overviews the optimization-based problem in truth discovery that we adopt for CPT aggregation. The basic idea is that the inferred truth is likely to be correct if a reliable source provides it. Therefore, the goal is to minimize the overall distance of the aggregated truth to a reliable source~\cite{CRH14}. Based on this principle, the optimization framework is defined as follows:

\small{
\begin{align*}
\min_{\mathcal{X}^*,\mathcal{W}} f(\mathcal{X}^*,\mathcal{W})=\sum_{k=1}^{K} {w_k} \sum_{i=1}^{N} \sum_{m=1}^{M} d_m(v_{im}^*,v_{im}^k)\\
s.t. \ {\delta(\mathcal W)=1 , \mathcal{W}\in \mathcal S}, \numberthis
\label{optimization}
\end{align*}
}
where $\mathcal{X}^*$ and $\mathcal{W}$ correspond to the set of truths and the source weight, respectively, and $w_k$ refers to the reliability degree of the $k$-th source. The function $d_m(\cdot,\cdot)$ measures the distance between the sources' observations $v_{im}^k$ and the aggregated truths $v_{im}^*$. The regularization function $\delta(\mathcal{W})$ is defined to guarantee the weights are always non-zero and positive.

To optimize the objective function Eq.~(\ref{optimization}), the block coordinate descent algorithm is applied by iteratively updating the aggregated truths and source weights, conducting the following two steps.

\textbf{Source Weight Update.}
    To update the source weight in the model, the values for the truths are considered fixed, and the source weights are computed, which jointly minimizes the objective function as shown in Eq.~(\ref{Weight Update}).

\small{
\begin{align*}
     \mathcal{W} \leftarrow \operatorname*{argmin}_\mathcal{W} f(\mathcal{X}^*,\mathcal{W})\  s.t. \  {\delta(\mathcal W)=\sum_{k=1}^{K}exp(-w_k)}. \numberthis
\label{Weight Update}
\end{align*}
}
This function regularizes the value of $w_k$ by constraining the sum of $exp({-w_k})$.

\textbf{Truth Update.}
    At this step, the weight of each source $w_k$ is fixed, and the truth is updated for each entry to minimize the difference between the truth and the sources' observations, where sources are weighted by their reliability degrees.

\small{
\begin{align*}
{v_{im}^{(*)}} \leftarrow
\operatorname*{argmin}_{v} 
{\sum_{k=1}^{K}w_k\cdot d_m(v,v_{im}^k)}. \numberthis
\label{Truth Update}
\end{align*}
}
By deriving the truth using Eq.~(\ref{Truth Update}) for every instance, the collection of truths $\mathcal{X}^*$ which minimizes $f(\mathcal{X}^*,\mathcal{W})$ with fixed $\mathcal{W}$ is obtained.

\begin{table}
  \resizebox{\columnwidth}{!}{%
  \begin{tabular}{c|cl}
    \hline
    \textbf{Notation} & \textbf{Definition}\\
    \hline
\ $n$ & number of sentences indexed by i\\ 
\ $p$ & number of parsers indexed by k\\
\ $S_i$ & the $i$-th sentence in the dataset\\
\ $\mathcal{W}$ & set of input CPTs' weights\\
\ $w_k^S$ & the weight of the $k$-th parser w.r.t. the clusters\\
\ $w_k^l$ & the weight of the $k$-th parser w.r.t. the labels\\
\ $T_{ik}$ & the $k$-th input CPT for the $i$-th sentence \\ 
\ $C_i$ & set of all unique clusters from input trees for the $i$-th sentence \\
\ $T_{i}^{S^*}$ & aggregated tree for the $i$-th sentence w.r.t. the tree structure \\
\ $T_{i}^*$ & aggregated tree for the $i$-th sentence w.r.t. the labels \\
\ $\mathcal{L}_{Clu(T)}$ & clusters' labels in tree T \\
\hline
\end{tabular}%
}
\caption{Summary of Notations}
\label{Notations}
\end{table}

\section{Constituency Parse Tree Aggregation Model (CPTAM)}
In this section, we first formally define the problem. Then, we propose our solution in two steps. In the first step, we focus on tree structure aggregation to resolve the conflict between input trees and obtain the aggregated tree structure $T_i^{S^*}$ of the CPTs. In the second step, the corresponding POS tags and constituent labels are obtained through the label aggregation. It is worth mentioning that both tree structures and tree labels are essential for adequately parsing a sentence.

\subsection{Problem Definition}
We define the CPT aggregation problem using the notations summarized in Table \ref{Notations}. Suppose there is a corpus that consists of $n$ sentences indexed by $i$ ($i\in[1,n]$), and $p$ different parsers indexed by $k$ ($k \in [1,p]$) produce CPTs for each sentence in the corpus. We use $T_{ik}$ to denote the $k$-th input CPT for the $i$-th sentence ($S_i$). Each input constituency parser has two weight parameters $w_k^S$ and $w_k^l$ to reflect the parser's reliability with respect to structure and constituent labels, respectively. We use different weight parameters for structure and constituent labels to account for the scenarios where a parser can successfully identify the phrase structure but assign incorrect labels. A higher weight implies that the parser is of higher reliability. The CPT aggregation problem seeks an aggregated tree for a sentence ($T_{i}^{*}$) given the input CPTs ($T_{ik}$), and estimates the qualities of parsers in the absence of ground truth.

\subsection{Tree Structure Aggregation}
We formulate our framework utilizing the truth discovery framework presented in Eq.~(\ref{optimization}). In the tree structure aggregation step, our goal is to minimize the overall weighted distance of the aggregated CPT ($T_{i}^{S^*}$) to the reliable input CPT ($T_{ik}^S$) considering the structure only. Various distance measurements can be plugged in the optimization function shown in Eq.~(\ref{optimization}). We adopt $RF$ distance defined in Eq.~(\ref{RF}).

{\textbf{Robinson-Foulds ($RF$) distance}} is a symmetric difference metric to calculate the distance between leaf-labeled trees \cite{robinson1981comparison} in terms of clusters, where a \emph{cluster} refers to a maximal set of leaves with a common ancestor in a rooted tree ($T$)~\cite{semple2003phylogenetics}. For any two trees $T_1$ and $T_2$ that share the same leaf set, the $RF$ distance is defined in Eq.~(\ref{RF}):  
\small{
\begin{align*}
RF(T_1, T_2)=|Clu(T_1)\Delta Clu(T_2)|, \numberthis
\label{RF} 
\end{align*}
}
where the operation $\Delta$ computes the symmetric difference between two sets (i.e., $A \Delta B = (A\backslash B) \bigcup (B\backslash A)$), function $|\cdot|$ computes the cardinality of the set, and $Clu(T)$ refers to the cluster set of tree $T$. 
Different from Tree Edit Distance (TED) \cite{schwarz2017new}, which takes $O(y^3)$ time \cite{demaine2009optimal} to calculate, where $y$ refers to the number of tokens in the sentence, $RF$ distance can be calculated in $O(|Clu(T_1)| + |Clu(T_2)|)$ time \cite{day1985optimal}.

Applying the truth discovery framework (Eq.~(\ref{optimization})), we formulate the CPT aggregation problem with respect to the tree structure as shown in Eq.~(\ref{optimization2}). Each parser has a weight parameter $w_k^{S}$ to reflect the reliability of that parser in terms of the structure, and $\mathcal{W^S}=\{w_1^S,w_2^S,...,w_p^S\}$ refers to the set of all parsers' weights in terms of the structure. The higher the weight, the more reliable the parser. The aggregated tree $\mathcal{T}^{S^*}$ is the one that can minimize the overall weighted RF distances.
\small{
\begin{align*}
    \min_{\mathcal{T}^{S^*}, \mathcal{W}^{S}} f(\mathcal{T}^{S^*}, \mathcal{W}^{S}) = \sum_{k=1}^{p}{w_k^{S}} \sum_{i=1}^{n}RF(T_{i}^{S^*}, T_{ik}^{S}). \numberthis
\label{optimization2}
\end{align*}
}
We follow the block coordinate descent method introduced in Section \ref{Preliminaries}. 
To update the weights of the input constituency parsers in the objective function Eq.~(\ref{optimization2}), $T_{i}^{S^*}$ is fixed, and $w_k^{S}$ is updated as follows:
\small{
\begin{align*}
   w_k^{S}= -log(\frac{\sum_{i} RF(T_{i}^{S^*},T_{ik}^S)}{max_k\sum_{i} RF(T_{i}^{S^*},T_{ik}^S)}). \numberthis
\label{Eq:ws}
\end{align*}
}
This means that the weight of a parser is inversely proportional to the maximum sum of the distance between its input trees (we use $T_{ik}^S$ to refer to the input CPT with respect to the structure) and the aggregated trees.
Next, we update the aggregated parse tree for each sentence to minimize the difference between the aggregated parse tree and the input CPTs by treating the weights as fixed. The aggregated tree is updated following Eq.~(\ref{Truth Update}) as shown in Eq.~(\ref{tree update}):
\small{
\begin{align*}
  {T_{i}^{S^*}}\xleftarrow\ 
  \operatorname*{argmin}_{T_{i}^{S^*}}
  & \sum_{k=1}^{p}w_k^{S} \sum_{i=1}^{n}
  {RF(T_{i}^{S^*},T_{ik}^S)}.  \numberthis 
\label{tree update}
\end{align*} 
}
We propose an optimal solution for Eq.~(\ref{tree update}). 
\subsubsection{The Optimal Solution}
We present an \textit{optimal} solution to obtain an aggregated tree by solving the optimization problem in Eq.~(\ref{tree update}). Our proposed approach constructs the aggregated tree by adding clusters with weighted support greater than or equal to 50\%, where support refers to the aggregated weight of CPTs containing that cluster. To establish the solution, we first demonstrate some properties of an optimal aggregated tree.
\begin{lemma}
The cluster set ($Clu(T_{i}^{S^*})$)
in Eq. (\ref{tree update}) satisfies the constraint $Clu(T_{i}^{S^*})\subseteq C_i$
($C_i=\cup_{k=1}^{p}Clu(T_{ik}^{S})$).
\label{lemma:1}
\end{lemma}
\begin{proof}
We can prove this lemma by contradiction. Suppose $ Clu(T_{i}^{S^*})$ is the optimal solution to Eq.~(\ref{tree update}) and there exists a cluster $c\neq \emptyset$ 
such that $c\in Clu(T_{i}^{S^*})$ but $c \notin C_i$. Therefore,  $c \notin T^{S}_{ik}, \forall k$. Let $Clu({{T}_{i}'}^{S^*}) = Clu({T_{i}}^{S^*}) - c$. Then based on the definition of RF distance, we have $\sum_{k=1}^{p} w_{k}^{S} \sum_{i=1}^{n} RF({T_{i}}^{{S}^*},T_{ik}) > \sum_{k=1}^{p} w_k^{S} \sum_{i=1}^{n} RF({{T}_{i}'}^{{S}^*},T_{ik}^{S})$, which contradicts the assumption that $ Clu(T_{i}^{S^*})$ is the optimal solution.
\label{sec:appendix}
\end{proof}

This property suggests that the search space of the solution to Eq.~(\ref{tree update}) is $C_i$. That is, all clusters in the aggregated tree must be present in at least one of the input CPTs.
\begin{lemma}
For any cluster $c$, if $\sum_{k=1}^p w_k^{S} \mathds{1}(c\in T_{ik}^S)> 0.5* \sum_{k=1}^p w_k^{S} $, then $c\in  Clu(T_{i}^{S^*})$, and if $\sum_{k=1}^p w_k^{S} \mathds{1}(c\in T_{ik}^S) < 0.5* \sum_{k=1}^p w_k^{S} $, then $c\notin  Clu(T_{i}^{S*})$, where $\mathds{1}(\cdot)$ is the indicator function.
\label{lemma:2}
\end{lemma}
\begin{proof}
The proof is similar to the proof for Lemma \ref{lemma:1}. We can prove the two statements by contradiction.
\end{proof}
Therefore, for the optimal solution, the clusters that have more than 50\% weighted support from all the input CPTs should be included in the aggregated tree. 

\begin{lemma}
For any cluster $c_1$ and $c_2$, if $\sum_{k=1}^p w_k^{S} \mathds{1}(c_1\in T_{ik}^{S})> 0.5* \sum_{k=1}^p w_k^{S} $ and $\sum_{k=1}^p w_k^{S} \mathds{1}(c_2\in T_{ik}^{S})> 0.5* \sum_{k=1}^p w_k^{S} $, then $c_1$ and $c_2$ must be compatible.
\label{lemma:3}
\end{lemma}
\begin{proof}
Note that for any constituency parse tree $T_{ik}^{S}$, its clusters must be compatible. Therefore, for a cluster $c$, all its non-compatible clusters can only occur in trees that $c$ is not occurred. If $\sum_{k=1}^p w_k^{S} \mathds{1}(c\in T_{ik}^{S})> 0.5* \sum_{k=1}^p w_k^{S} $, then $\forall c'$ not compatible with $c$, $\sum_{k=1}^p w_k^{S} \mathds{1}(c'\in T_{ik}^{S})< 0.5* \sum_{k=1}^p w_k^{S} $, and based on Lemma \ref{lemma:2}, $c'\notin Clu(T_{i}^{S^*})$.
\end{proof}

There is a special case when $\sum_{k=1}^p w_k^{S} \mathds{1}(c\in T_{ik}^S) = 0.5* \sum_{k=1}^p w_k^{S} $. To consider this situation, we add the compatibility constraint as follows:
\small{
\begin{align*}
c_1\cap c_2=\emptyset,
or\ c_1\subset c_2,\ or\ c_2\subset c_1,
\forall c_1, c_2 \in Clu(T_{i}^{S^*}). \numberthis
\label{C2}
\end{align*}
}
This constraint ensures that the aggregated tree follows the syntactic structure requirement of constituency parsing. Therefore, all the clusters in the aggregated tree should be \textit{compatible}, which means they should either be disjoint or a proper subset.

In the cases where $\sum_{k=1}^p w_k^{S} \mathds{1}(c\in T_{ik}^S) = 0.5* \sum_{k=1}^p w_k^{S} $, we propose to find the maximum number of compatible clusters to add into the aggregated tree $Clu(T_{i}^{S^*})$. Although adding these clusters into the constructed aggregation tree does not affect the resulting total $RF$ distance, we favor the aggregated trees with more compatible clusters since they contain as many details from the input trees. We conduct the following steps. First we form a set $C'_i$ that includes all clusters such that $\sum_{k=1}^p w_k^{S} \mathds{1}(c\in T_{ik}^S) = 0.5* \sum_{k=1}^p w_k^{S} $. Then we construct the incompatibility graph by treating the clusters as nodes and adding an edge if two clusters are not compatible. Finding the maximum number of compatible clusters is then equivalent to the maximum independent set problem \cite{tarjan1977finding}. This strong NP-hard problem can be addressed by the existing methods \cite{liu1993maximal,kashiwabara1992generation}.

Based on the properties of the optimal solution, we construct the aggregated tree $T_{i}^{S^*}$ as follows. We compute the weighted support for each cluster $c$ in $C_i$. If $\sum_{k=1}^p w_k^{S} \mathds{1}(c\in T_{ik}^S)> 0.5* \sum_{k=1}^p w_k^{S}$ then $c$ is added to the aggregated tree $Clu(T_{i}^{S^*})$. If $\sum_{k=1}^p w_k^{S} \mathds{1}(c\in T_{ik}^S) = 0.5* \sum_{k=1}^p w_k^{S} $ then we find maximum number of compatible clusters $C_i^m$ by solving the maximum independent set problem. We then add these clusters to the aggregated tree $Clu(T_{i}^{S^*})$. Finally, we re-order the clusters in $Clu(T_{i}^{S^*})$ to form $T_{i}^{S^*}$. The pseudo-code of our proposed algorithm is given in Algorithm (\ref{alg:Medcpt}). 
\begin{algorithm}  \small
\caption{Optimal solution to Eq. (\ref{tree update})} 
\label{alg:Medcpt}
\begin{algorithmic} 
\Require The set of unique clusters in all input CPTs for i-th sentence ($C_i$), weights ($w_{k}^{S}$).
\Ensure Aggregated CPT ($T_{i}^{S^*}$).
\State $Clu(T_{i}^{S^*})= \emptyset$ ;
\State $C'_i =\emptyset$;
\For{$c$ in $C_i$}
\If{$\sum_{k=1}^p w_k^{S} \mathds{1}(c\in T_{ik}^S) > 0.5* \sum_{k=1}^p w_k^{S} $} 
    \State $Clu(T_{i}^{S^*}) = Clu(T_{i}^{S^*})\bigcup c$\;
\EndIf
\If{$\sum_{k=1}^p w_k^{S} \mathds{1}(c\in T_{ik}^S) = 0.5* \sum_{k=1}^p w_k^{S} $}
    \State $C'_i=C'_i\bigcup c$;
\EndIf
 \EndFor
\State Construct incompatibility graph $g$ for $C'_i$;

\State $C^m_i$ = Maximum-Independent-Set($g$);

\State {$Clu(T_{i}^{S^*}) = Clu(T_{i}^{S^*})\bigcup C^m_i$\;}

\State \Return {$T_{i}^{S^*}$}\;

\end{algorithmic}
\end{algorithm}
\begin{theorem} \label{theorem1}
The aggregated tree $T_{i}^{S^*}$ calculated by Algorithm (\ref{alg:Medcpt}) is the optimal solution to the following problem:
\small{
\begin{align*}
      {T_{i}^{S^*}} \xleftarrow\ 
      \operatorname*{argmin}_{T_{i}^{S^*}}
      \sum_{k=1}^{p}w_k^{S} \sum_{i=1}^{n}
      {RF(T_{i}^{S^*},T_{ik}^{S})} \nonumber
\end{align*}
}
such that
\small
\begin{align*}
 c_1\cap c_2=\emptyset, 
      or\ c_1\subset c_2,\ or\ c_2\subset c_1, \forall c_1, c_2 \in Clu(T_{i}^{S^*}).
\end{align*}
\end{theorem}
\begin{proof} In the Algorithm (\ref{alg:Medcpt}), we consider all the clusters with $\sum_{k=1}^p w_k^{S} \mathds{1}(c'\in T_{ik}^{S})> 0.5* \sum_{k=1}^p w_k^{S}$ and add them to $T_{i}^{S^*}$. From Lemma (\ref{lemma:3}), we show that all of these clusters are compatible and from Lemma (\ref{lemma:2}), we show that adding these clusters minimizes the RF distance. Adding all these clusters result in the minimum RF distance implying that the objective function will be minimized. Applying maximum independent set algorithm on the incompatibility graph provides us with the maximum number of compatible clusters for $\sum_{k=1}^p w_k^{S} \mathds{1}(c'\in T_{ik}^{S})= 0.5* \sum_{k=1}^p w_k^{S}$. Adding all these clusters to $T_{i}^{S^*}$ results in obtaining the maximum set of compatible clusters. Thus the solution is optimal. 
\end{proof}
\subsubsection{Time Complexity}
\begin{lemma} 
The incompatibility graph constructed for clusters with weighted support equal to 50\% is bipartite.
\label{Lemma 4}
\end{lemma}  
\begin{proof} Let $C_1, \ldots,C_k$ be the clusters with 50\% support from the input constituency parse trees. The set of all constituency parse trees with respect to structure is denoted by $\mathcal{T}^{S}=\{T_{i1}^{S}, T_{i2}^{S}, T_{i3}^{S},...,T_{ik}^{S}\}$.

Assume that a cluster $C_i$ is supported by trees $\cT_{i}^{S} \subset \mathcal{T}^{S}$. If $C_i$ is not compatible with $C_j$, then it implies that $\cT_{i}^{S}= \mathcal{T}^{S} \setminus \cT_{j}^{S}$. Otherwise, $\cT_{i}^{S}$ would have a non-empty intersection with $\cT_{j}^{S}$, which would imply that there is a tree $\cT_{s}^{S}$ that supports both $C_i$ and $C_j$, which contradicts with the assumption that $C_i$ and $C_j$ are incompatible.

We prove Lemma \ref{Lemma 4} by contradiction. Let's assume that the incompatibility graph $G$ for clusters $C_1, \ldots, C_k$ is \emph{not bipartite}. It means that $G$ contains an odd-cycle. Without loss of generality assume that this cycle is $(C_1, C_2, \ldots, C_{2p+1})$. That is, $C_2$ is not compatible with $C_1$, $C_3$ is not compatible with $C_2$, and so on.
Then, by our previous observation, $C_2$ must be supported by $\cT^{S} \setminus \cT_1^{S}$, $C_3$ must be supported by $\cT_1^{S}$, and so on. That is, for odd $i \le 2p+1$, $C_i$ must be supported by $\cT_1^{S}$. Then $C_{2p+1}$ and $C_1$ are supported by the same set of trees, which means that $C_1$ and $C_{2p+1}$ must be compatible. This is a contradiction (i.e., $(C_1, C_2, \ldots, C_{2p+1})$ could not be a cycle).
\end{proof}

The existing methods \cite{kashiwabara1992generation} solve the maximum independent set problem for a bipartite graph with time complexity of $O(z^{2.5}+(output size))$ where $z$ refers to the number of nodes in the incompatibility graph. As the expected output is the list of compatible clusters, the output size is in the order of $O(z)$. The for loop that iterates over cluster set $C_i$ runs in $O(|C_i|)$ time. Therefore, the overall run time of Algorithm (\ref{alg:Medcpt}) is $O(|C_i| + z^{2.5}+z)$. However, in practice, $z$ is very small compared to $|C_i|$ because it only contains clusters with support equal to 50\%. Thus, Algorithm (\ref{alg:Medcpt}) has, in practice, near-linear run time in $|C_i|$. 
\subsection{Constituent Label Aggregation} 

After obtaining the aggregated structures, we aggregate the corresponding labels provided by the parsers. In constituent label aggregation step, we aim to minimize the objective function Eq.~(\ref{optimization-label}) with respect to the $\mathcal{L}_{Clu{(T_{i}^{S^*}})}$ and $\mathcal{W}^l$, where $\mathcal{L}_{Clu{(T_{i}^{S^*}})}$ refers to the labels associated to the aggregated structure, and $\mathcal{W}^l = \{w_1^l, w_2^l,...,w_p^l\}$ refers to the set of all parsers’ weights with respect to the constituent labels, as follows:
\small{
\begin{align*}
    \min_{\mathcal{T}^*,\mathcal{W}^l} f(\mathcal{T}^*,\mathcal{W}^l)=\sum_{k=1}^{p}{w_k^{l}}\sum_{i=1}^{n}d(\mathcal{L}_{Clu({T_{i}^{S^*}})}, \mathcal{L}_{Clu({T_{ik}^{S}})}), \numberthis
\label{optimization-label}
\end{align*}
}
where $\mathcal{L}_{Clu({T_{ik}^{S})}}$ refers to the constituent labels provided by parsers for the obtained clusters in $T_{i}^{S^*}$. Accordingly, we show the weight update by taking differentiation with respect to $\mathcal{W}^l$ in Eq.~(\ref{Eq:wl}):
\small{
\begin{align*}
  w_k^l= -log(\frac{\sum_{i} d(\mathcal{L}_{Clu({T_{i}^{S^*}})}, \mathcal{L}_{Clu({T_{ik}^{S}})})}{max_{k}\sum_{i} d(\mathcal{L}_{Clu({T_{i}^{S^*}})}, \mathcal{L}_{Clu({T_{ik}^{S}})})}), \numberthis
\label{Eq:wl}
\end{align*}
}
where $d$ refers to the zero-one loss function. Similarly, the constituent label aggregation update is shown in Eq.~(\ref{label update}):
\small
\begin{align*}
      { 
      {T_{i}^{*}}\xleftarrow\
      \operatorname*{argmin}_{\mathcal{L}_{Clu(T_{i}^{S^*})}}}
      \sum_{k=1}^{p}w_{k}^l \sum_{i=1}^{n}
      d(\mathcal{L}_{Clu({T_{i}^{S^*}})}, \mathcal{L}_{Clu({T_{ik}^S)}}). \numberthis
 \label{label update}
 \end{align*}

\section{Experiments}
In this section, we conduct experiments on various datasets with different languages from different domains\footnote{Our implementation code is available at \url{https://github.com/kulkarniadithya/CPTAM}}. We start with the datasets in Section \ref{dataset}.The baseline methods and evaluations are discussed in Sections \ref{Baselines} and \ref{evaluations}, respectively. We demonstrate the main experimental results in Section \ref{experimental : label results} and ablation studies in Section \ref{sec:ablation}. 

\begin{table} 
\resizebox{\columnwidth}{!}{%
\centering
\begin{tabular}{l|c|c| c c c c c|} 
\hline
\textbf{Datasets} & Language & Sentence & \#token/sentence & \\  
\hline
\hline
\textbf{Penn Treebank-3} & English & 49208 & 24.70 & \\ 
\hline
\textbf{OntoNotes} & English & 143709 & 18.59 & \\
                     & Chinese & 51230 & 17.64 & \\
\hline
\textbf{Genia} & English & 18541 & 28.09 \\
\hline
\textbf{TIGER Corpus}  & German & 40020 & 17.06 & \\
\hline
\textbf{French Treebank}    & French & 21550 & 24.80 \\
\hline

\end{tabular}%
    }
\caption{Statistics of Datasets} \label{Table Dataset}
\end{table}
\subsection{Datasets} \label{dataset}
We use six benchmark datasets from different domains and different languages for evaluation. 

\textbf{Penn Treebank-3}\footnote{\url{https://catalog.ldc.upenn.edu/LDC99T42}} selected 2,499 stories from a three-year Wall Street Journal collection in English for syntactic annotation.

\textbf{OntoNotes}\footnote{\url{https://catalog.ldc.upenn.edu/LDC2013T19}} consists of a large corpus comprising various genres of text (e.g., news, weblogs, Usenet newsgroups, broadcast, and talk shows) with structural information in three languages (English, Chinese, and Arabic). The Arabic portion of the dataset is not included in our experiments since the parsers' tokenization does not align with the ground truth.

\textbf{Genia}\footnote{\url{https://github.com/allenai/genia-dependency-trees/tree/master/original_data}} is constructed from research abstracts in the molecular biology domain. Approximately 2500 abstracts are annotated from the MEDLINE database.

\textbf{TIGER Corpus}\footnote{\url{https://www.ims.uni-stuttgart.de/documents/ressourcen/korpora/tiger-corpus/download/start.html}} consists of approximately 40,000 sentences from the German newspaper "Frankfurter Rundschau". The corpus was annotated with part-of-speech and syntactic structures in the project TIGER (DFG).

\textbf{French Treebank}\footnote{\url{http://ftb.llf-paris.fr/telecharger.php?langue=en}} consists of approximately 22000 sentences from the articles of French newspaper "Le Monde". \\
Table \ref{Table Dataset} summarizes the statistics of the datasets. 

\subsection{Baselines} \label{Baselines}
We compare CPTAM with two categories of baselines. The first category of baselines is the individual state-of-the-art input constituency parsers including \textbf{CoreNLP} \cite{manning2014stanford}, \textbf{Berkeley}\footnote{We use the pretrained model provided by spaCy} \cite{kitaev-klein-2018-constituency}, \textbf{AllenNLP\footnote{This parser can parse sentences in English only.}} \cite{gardner-etal-2018-allennlp}, and \textbf{HanLP} \cite{hanlp2}. We have chosen these parsers as they are the most ``stars'' NLP libraries on GitHub, demonstrating their wide applications in industry and academia.

The second category of baselines is the tree aggregation methods\footnote{We apply the implementations from \url{https://evolution.genetics.washington.edu/phylip/getme-new1.html}} including
\begin{itemize}
    \item \textbf{Majority Rule Consensus (MRC)} \cite{MARGUSH1981239}. It constructs aggregation trees containing clusters with support greater than $50\%$.
    \item \textbf{Greedy Consensus (GC)} \cite{bryant2003classification}. The aggregated trees are constructed progressively to have all the clusters whose support is above a threshold ($30\%$ for OntoNotes Chinese,TIGER Corpus, and French Treebank, and $20\%$ for the other datasets) and compatible with the constructed tree. With these thresholds, this baseline essentially constructs aggregation trees with all compatible clusters from input trees.
    \item \textbf{Strict Consensus (SC)} \cite{bryant2003classification}. It constructs aggregation trees containing clusters with support of $100\%$. 
\end{itemize}
These methods only consider the aggregation of tree structures but not labels. Therefore, we apply \textbf{Majority Voting (MV)} to aggregate the labels after the tree aggregation step, where the label with the highest frequency is chosen for each cluster. We also compare with \textbf{CPTAM-W}, which is \textbf{CPTAM} without weight estimation. CPTAM-W considers clusters with support greater than or equal to $50\%$; thus, it is more aggressive compared to MRC, which considers clusters with support greater than $50\%$ only, and more conservative compared to GC, which includes all compatible clusters.

\begin{table*} [!htbp] \small
\centering 
\begin{adjustbox}{width=160mm,center}
\begin{tabular}{l| c c c | c c c | c c c | c c c}
\hline
& \multicolumn{3}{c|}{\text{Penn Treebank-3}}   
& \multicolumn{3}{c|}{\text{OntoNotes (English)}} 
& \multicolumn{3}{c|}{\text{OntoNotes (Chinese)}}
& \multicolumn{3}{c}{\text{Genia}}\\
&  \text{$P$} & \text{$R$} & \text{$F1$} &  \text{$P$} & \text{$R$} & \text{$F1$} &  \text{$P$} & \text{$R$} & \text{$F1$} & \text{$P$} & \text{$R$} & \text{$F1$} \\
\hline 
CoreNLP & 81.35 & 83.16 & 82.25 & 77.68 & 77.66 & 77.67 & 86.58 & 86.77 & 86.67 & 66.38 & 72.46 & 69.29 
\\
Berkeley & 91.13 & 93.27 & 92.19 & 80.31 & 79.04 & 79.67 & 77.66 & 69.93 & 73.59 & 69.49 & 72.82 & 71.12
\\
AllenNLP & 93.08 & \textbf{94.96} & \textbf{94.01} & 80.99 & 80.47 & 80.73 & - & - & - & 68.00 & 69.54 & 68.76
\\
Hanlp & 18.42 & 19.21 & 18.81 & 21.61 & 21.90 & 21.75 & \textbf{92.56} & \textbf{92.63} & \textbf{92.59} & 28.65 & 29.19 & 28.92
\\
\hline
MRC + MV & 89.75 & 89.79 & 89.77 & 78.11 & 73.75 & 75.87 & 89.05 & 89.07 & 89.06 & 68.92 & 69.12 & 69.02
\\
GC + MV & 88.76 & 90.19  & 89.47 & 76.85 & 75.51 & 76.17 & 87.90 & 89.99 & 88.93 &  66.22 & 70.36 & 68.22
\\
SC + MV & 90.46 & 85.14 & 87.72 & 78.95 & 66.72 & 72.32 & 89.79 & 87.78 &  88.77 & 70.82 & 59.15 & 64.46
\\
\hline
\textbf{CPTAM-W} & 90.79 & 90.13 & 90.46 & 78.87 & 74.87 & 76.82 & 89.19 & 89.11 & 89.15 & 70.02 & 69.92 & 69.97
\\

\textbf{CPTAM} & \textbf{93.58} & 93.46 & 93.52 & \textbf{81.94} & \textbf{80.72} &  \textbf{81.33} & 91.55 & 91.47 & 91.51 & \textbf{71.43} & \textbf{72.43} & \textbf{71.93}
\\
\hline
\end{tabular}
\end{adjustbox}
\caption{CPT aggregation performance comparison on Penn Treebank-3, OntoNotes (English, Chinese), and Genia datasets. 
} 
\label{table: Results 1}
\end{table*}

\begin{table*} [!htbp]
\huge
\centering 
\begin{adjustbox}{width=170mm,center}
\begin{tabular}{l| c c| c c | c c| c c | c c |c c | c c| c c | c c | c c}
\hline
& \multicolumn{4}{c|}{\text{Penn Treebank-3}}   
& \multicolumn{4}{c|}{\text{OntoNotes (English)}} 
& \multicolumn{4}{c|}{\text{Genia}}
& \multicolumn{4}{c|}{\text{OntoNotes (Chinese)}}   
& \multicolumn{2}{c|}{\text{French Treebank}} 
& \multicolumn{2}{|c}{\text{TIGER Corpus}} \\

& $F1^S$ &  $w^S$ & $Acc^l$ &  $w^l$ & $F1^S$ & $w^S$ & $Acc^l$ &  $w^l$ & $F1^S$ &  $w^S$ & $Acc^l$ & $w^l$ & $F1^S$ &  $w^S$ & $Acc^l$ &  $w^l$ & $F1^S$ &  $w^S$ & $F1^S$ &  $w^S$ \\
\hline 
CoreNLP & 3 & 3 & 3 & 3 & 3 & 3 & 3 & 3 & 3 & 3 & 1 & 1 & 1 & 1 & 2 & 2 & 1 & 1 & 2 & 2
\\
Berkeley & 2 & 2 & 2 & 2 & 2 & 2 & 2 & 2 & 1 & 1 & 2 & 2 & 3 & 3 & 3 & 3 & 2 & 2 & 1 & 1 
\\ 
Allennlp & 1 & 1 & 1 & 1 & 1 & 1 & 1 & 1 & 2 & 2 & 3 & 3 & - & - & - & - & - & - & - & -
\\ 
HanLP & 4 & 4 & 4 & 4 & 4 & 4 & 4 & 4 & 4 & 4 & 4 & 4 & 2 & 2 & 1 & 1 & 3 & 3 & 3 & 3
\\ 
\hline
\end{tabular}
\end{adjustbox}
\caption{The comparison between the rankings of parsers' performance with the rankings of estimated weights}   
\label{table: Results 4}
\end{table*}

\subsection{Evaluation Measurements} \label{evaluations}
The performance is evaluated by different standard metrics in the experiments. To evaluate the performance based on the real-life usage of constituency parsers, we also include the POS tags of individual tokens as part of the parsing results. 
Therefore, the following evaluation metric is stricter than \textit{Evalb}, the standard metric for evaluating phrase structure. 
We report Precision, Recall, and F1 as follows: 
\small{
\begin{align}
    Precision (P)&=\frac{\# Correct \ Constituents}{\#Constituents \ in \  parser \ output}\\
    Recall (R)&=\frac{\#Correct \  Constituents}{\#Constituents \ in \ gold \ standard}\\
    F1 &= \frac{2*Precision*Recall}{Precision+Recall}.
\end{align}
}
\noindent Accordingly, the same metrics Precision ($P^S$), Recall ($R^S$), and F1 ($F1^S$) are defined to evaluate the performance considering only the tree structure.

\subsection{Experimental Results} \label{experimental : label results}
The experimental results for CPT aggregation performance on Penn Treebank-3, OntoNotes (English and Chinese), and Genia are summarized in Table \ref{table: Results 1}. Our experiments consider the scenario where the users only have access to the parsers but do not have any prior knowledge about their performance. Since there is no prior knowledge about parser performance on different datasets or languages, we consider the freely available state-of-the-art parsers to obtain initial parsing results. For French Treebank and TIGER Corpus datasets, as the ground truth labels are different from the labels provided by the parsers, we do not consider them for constituent label aggregation\footnote{Taking CoreNLP output as an example in TIGER corpus, a chunk of the sentence is tagged as \textit{(NUR (S (NOUN Konzernchefs) (VERB lehnen)} while the ground truth for the same span is \textit{(NN-SB Konzernchefs) (VVFIN-HD lehnen)}}.

Among the input parsers, it is clear that none of them consistently perform the best on all datasets. Specifically, Hanlp performs poorly on English but the best on Chinese. This may be caused since Hanlp targets the Chinese language even though it is software for multiple languages. Allennlp performs the best on Penn Treebank-3 and OntoNotes (English) among the four parsers but does not perform well on the Genia dataset in the biomedical domain.

The proposed CPTAM significantly outperforms all the state-of-the-art aggregation methods in terms of Precision, Recall, and F1 score, demonstrating the power of the proposed method in resolving the conflicts between input CPTs.  
Comparing CPTAM-W and CPTAM, CPTAM further improves in all metrics, indicating the necessity and effectiveness of the weight estimation in the truth discovery process. 

Compared with individual parsers, CPTAM performs consistently well to represent the constituency grammar of the sentence in all datasets. CPTAM performs the best for two out of four datasets and remains competitive on the other two datasets. In contrast, AllenNLP and Hanlp are the best for one out of four datasets, and CoreNLP and Berkeley are not the best in any datasets. This shows that the proposed CPTAM can consistently obtain high-quality aggregated CPTs over different languages and domains. 

Further, we study the accuracy of the weight estimations of CPTAM. We compare the rankings given by the estimated weights with the rankings of parsers' real quality, and the results are shown in Table \ref{table: Results 4}. To evaluate the weights estimated for the structure aggregation, we compute the rank of parsers' quality by their structure F1 scores ($F1^S$) compared with the ground truth and by the weight estimation $w_{k}^S$ computed in Eq.~(\ref{Eq:ws}), where the numbers indicate the rank. Similarly, for the label aggregation, we compute the rank of parsers' quality by their label accuracy ($Acc^l$) and by the weight estimation $w_{k}^l$ computed in Eq.~(\ref{Eq:wl}).

It is clear that the parsers' quality varies across different languages and domains. The ranks of parsers are exactly the same between their real quality and the estimated weights. It illustrates that the weight calculated by the proposed CPTAM properly estimates parsers' quality in the absence of ground truth. These experiments also suggest that parser users can first apply CPTAM on the sampled corpus to estimate the quality of individual parsers on the given corpus and then use the best parser to achieve high-quality parsing results and high efficiency. 


\begin{table*} [!htbp]  \small
\centering 
\begin{adjustbox}{width=160mm,center}
\begin{tabular}{l| c c c c | c c c c | c c c c}
\hline
& \multicolumn{4}{c|}{\text{Penn Treebank-3}}   
& \multicolumn{4}{c|}{\text{OntoNotes (English)}} 
& \multicolumn{4}{c}{\text{Genia}} \\
& \text{$RF dist.$} &  \text{$P^{S}$} & \text{$R^{S}$} & \text{$F1^{S}$} &  \text{$RF dist.$} &  \text{$P^{S}$} & \text{$R^{S}$} & \text{$F1^{S}$} &
\text{$RF dist.$} &  \text{$P^{S}$} & \text{$R^{S}$} & \text{$F1^{S}$}
\\
\hline 
CoreNLP & 436250 & 84.76 & 86.76 & 85.75 & 890401 & 87.58 & 89.03 & 88.30 & 341925 & 77.75 & 85.70 & 81.53
\\
Berkeley & 194282 & 94.40 & 96.16 & 95.27 & 790028 & 89.66 & 89.88 & 89.77 & 273796 & 82.84 & \textbf{87.34} & 85.03
\\
AllenNLP & \textbf{154795} & 95.41 & \textbf{97.29} & \textbf{96.34} & 696054 & 90.55 & \textbf{91.75} & 91.15 & 298548 & 82.23 & 83.90 & 83.06
\\
Hanlp & 1709148 & 59.26 & 60.69 & 59.97 & 3119630 & 64.71 & 66.24 & 65.47 & 521522 & 69.26 & 69.50 & 69.38 
\\
\hline
MRC & 265967 & 93.23 & 92.79 & 93.01 & 1170233  & 87.98  & 83.18 & 85.51 & 307761 & 83.05 & 82.78 & 82.91
\\
GC & 266406 & 91.94 & 93.98 & 92.95 & 1028983 & 85.77 & 85.88 & 85.82 & 316758 & 78.93 & 84.16 & 81.46
\\
SC & 320215 & 93.77 & 87.53 & 90.54 & 1323468 & 88.39 & 75.33 & 81.34 & 370522 & 83.70 & 67.09 & 74.48
\\
\hline
\textbf{CPTAM-W} & 243413 & 93.12 & 93.94 & 93.53 & 1007384 & 87.85 & 84.63 & 86.21 & 293364 & 82.95 & 84.01 & 83.47 \\
\textbf{CPTAM} & 176187 & \textbf{95.92} & 95.73 & 95.82 & \textbf{641058} & \textbf{91.28} & 91.71 & \textbf{91.49} & \textbf{264153} & \textbf{84.31} & 85.96 & \textbf{85.13}      
\\
\hline
\end{tabular}
\end{adjustbox}
\caption{The tree structure aggregation performance comparison on Penn Treebank-3, OntoNotes (English), and Genia datasets}
\label{table: Results 2}
\end{table*}


\begin{table*} [!htbp]  \small
\centering 
\begin{adjustbox}{width=160mm,center}
\begin{tabular}{l| c c c c | c c c c | c c c c}
\hline
& \multicolumn{4}{c|}{\text{OntoNotes (Chinese)}}   
& \multicolumn{4}{c|}{\text{French Treebank}} 
& \multicolumn{4}{c}{\text{TIGER Corpus}} \\

& \text{$RF dist.$} &  \text{$P^{S}$} & \text{$R^{S}$} & \text{$F1^{S}$} &  \text{$RF dist.$} &  \text{$P^{S}$} & \text{$R^{S}$} & \text{$F1^{S}$} &
\text{$RF dist.$} &  \text{$P^{S}$} & \text{$R^{S}$} & \text{$F1^{S}$}
\\
\hline 
CoreNLP & \textbf{108817} & \textbf{96.47} & \textbf{96.72} & \textbf{96.59} & \textbf{212550} & \textbf{91.06} & \textbf{91.56} & \textbf{91.31} & 740059 & 65.67 & 80.02  & 72.14
\\
Berkeley & 406708 & 92.90 & 84.16 & 88.31 & 344070 & 85.65  & 77.07 & 81.13 & \textbf{220183} & \textbf{93.97} & \textbf{85.36} & \textbf{89.46} 
\\ 
HanLP & 144618 & 95.73 & 95.78 & 95.75 & 1196428 & 44.90 & 37.82 & 41.06 & 818677 & 66.22 & 62.36 & 64.23 
\\ 
\hline
MRC & 200008 & 95.41 & 94.04 & 94.72 & 231755 & 90.06 & 88.05 & 89.04 & 229139  & 92.95  & 82.79 & 87.58
\\
GC & 206438 & 93.93 & 95.06 & 94.49 & 308335 & 83.80 & 89.08 & 86.35 & 315613 & 87.66 & 84.38 & 85.99
\\
SC & 213173 & 95.83 & 92.84 & 94.31 & 314548 & 90.97 & 79.30 & 84.74 & 293991 & 93.07 & 80.23 & 86.17
\\
\hline
\textbf{CPTAM-W} & 198769 & 95.34 & 94.20 & 94.77 & 229344 & 90.8 & 88.35 & 89.56 & 228845 & 92.83 & 83.97 & 88.18 \\
\textbf{CPTAM} & 114733 & 96.39 & 96.49 & 96.44 & 212697 & 91.05 & 91.54 & 91.29 & \textbf{220183} & \textbf{93.97} & \textbf{85.36} & \textbf{89.46}\\
\hline
\end{tabular}
\end{adjustbox}
\caption{The tree structure aggregation performance comparison on OntoNotes (Chinese), French Treebank and TIGER Corpus datasets} 
\label{table: Results 3}
\end{table*}


\subsection{Ablation Study}\label{sec:ablation}
To gain insights into our framework, we investigate the effectiveness of the tree structure aggregation step as it is the foundation of CPTAM. 
To evaluate the performance on the structure, the $RF$ distance ($RF dist.$) is calculated between the parser output and ground truth. We also calculate Precision ($P^S$), Recall ($R^S$), and F1 score ($F1^S$) considering the tree structure only. The ablation study results are shown in Table \ref{table: Results 2} and Table \ref{table: Results 3}.

Table \ref{table: Results 2} and Table \ref{table: Results 3} illustrate a strong correlation between the $RF$ distance and F1 score on all datasets. The lower the $RF$ distance, the higher the F1 score. This correlation indicates that $RF$ distance is a proper measurement for the quality of constituency parse trees. CPTAM outperforms all aggregation baseline approaches on all datasets. It consistently identifies proper clusters in the tree by correctly estimating the parsers' quality. As a result, CPTAM outperforms or stays competitive compared to the best parser on all datasets as well.

\section{Conclusion}

This paper adopts the truth discovery idea to aggregate CPTs from different parsers by estimating the parsers' reliability in the absence of ground truth. We aggregate the input CPTs in two steps: tree structure aggregation and constituent label aggregation. The block coordinate descent method is applied to obtain solutions through an iterative process, and an optimal solution is derived to construct aggregation trees that can minimize the weighted $RF$ distance. We further provide theoretical analysis to show that the proposed approach gives the optimal solution. The proposed solution has near-linear run time in practice for the tree structure aggregation step. Our experimental results illustrate that the proposed solution CPTAM outperforms the state-of-the-art aggregation baselines and consistently obtains high-quality aggregated CPTs for various datasets in the absence of ground truth. We further illustrate that our adopted weight update correctly estimates parsers' quality. Empirically, the importance of the tree structure aggregation step is demonstrated in the ablation study. Overall, we present the effectiveness of the proposed CPTAM across different languages and domains.

\section{Acknowledgement}
The work was supported in part by the National Science Foundation under Grant NSF IIS-2007941. Any opinions, findings, and conclusions, or recommendations expressed in this document are those of the author(s) and should not be interpreted as the views of any U.S. Government. The U.S. Government is authorized to reproduce and distribute reprints for Government purposes, notwithstanding any copyright notation hereon.

\bibliography{CPTAM2.bib}
\bibliographystyle{plain}

\end{document}